\documentclass{article}
\usepackage{ijcai17}
\usepackage{times}

\usepackage{hyperref}
\usepackage{url}
\usepackage{amssymb}
\usepackage{amsmath}
\usepackage{graphicx}
\usepackage{chngcntr}
\usepackage[lined,algonl,boxed,oldcommands]{algorithm2e}
\usepackage{color}
\usepackage{subfigure}
\usepackage{lscape}
\usepackage[figuresleft]{rotating}
\usepackage{sidecap}
\usepackage[font=small]{caption}
\usepackage{amsthm}
\usepackage{booktabs, topcapt}
\usepackage{nicefrac}       % compact symbols for 1/2, etc.
\usepackage{microtype}      % microtypography

\newcommand{\be}{\begin{eqnarray} \begin{aligned}}
\newcommand{\ee}{\end{aligned} \end{eqnarray} }
\newcommand{\benn}{\begin{eqnarray*} \begin{aligned}}
\newcommand{\eenn}{\end{aligned} \end{eqnarray*} }
\newcommand{\vx}{{\mathbf x}}
\newcommand{\vf}{{\mathbf f}}
\newcommand{\vw}{{\mathbf w}}
\newcommand{\vr}{{\boldsymbol{\rho}}}
\newcommand{\vy}{{\mathbf y}}
\newcommand{\la}{{\big \langle}}
\newcommand{\ra}{{\big \rangle}}

\makeatletter
\def\thm@space@setup{%
  \thm@preskip=\parskip \thm@postskip=0pt
}
\makeatother

\newtheorem{theorem}{Theorem}[section]
\newtheorem*{theorem*}{Theorem}

\newtheorem{corollary}[theorem]{Corollary}

\title{Sifting Common Information from Many Variables}

\author{
Greg Ver Steeg, Shuyang Gao, Kyle Reing, Aram Galstyan \\
University of Southern California\\
Information Sciences Institute \\
gregv@isi.edu, gaos@usc.edu, reing@usc.edu, galstyan@isi.edu
}

\begin{document}
\maketitle

% Old, terrible abstract
%\begin{abstract}
%The information sieve was recently introduced as a way to incrementally and hierarchically extract the information in high-dimensional discrete data. 
%The sieve relies on an information-theoretic optimization that, until now, had no known solution for continuous variables.
%Under the restrictions of linearity and Gaussianity, we show that this optimization has a simple fixed point solution that exhibits exponential rates of convergence. 
%By exploiting invariances in the objective, the linear case can be directly extended to handle nonlinear and non-Gaussian scenarios. 
%The linear sieve complements other classic linear decompositions such as PCA and ICA, but, unlike those methods, the sieve is more scalable (linear instead of quadratic in the number of variables) and is intrinsically scale-invariant so that it does not require pre-whitening of the data.
%We demonstrate several uses of the sieve including exploratory data analysis and blind source separation in a scenario where ICA fails due to Gaussian sources.  
%\end{abstract}

\begin{abstract}
Measuring the relationship between any \emph{pair} of variables is a rich and active area of research that is central to scientific practice. In contrast, characterizing the common information among any \emph{group} of variables is typically a theoretical exercise with few practical methods for high-dimensional data.
A promising solution would be a multivariate generalization of the famous Wyner common information, but this approach relies on solving an apparently intractable optimization problem. 
We leverage the recently introduced information sieve decomposition to formulate an incremental version of the common information problem that admits a simple fixed point solution, fast convergence, and complexity that is linear in the number of variables. 
This scalable approach allows us to demonstrate the usefulness of common information in high-dimensional learning problems.
The sieve outperforms standard methods on dimensionality reduction tasks, solves 
a blind source separation problem that cannot be solved with ICA, and accurately recovers structure in brain imaging data. 
\end{abstract}

\section{Introduction} 

One of the most fundamental measures of the relationship between two random variables, $X_1, X_2$, is given by the mutual information, $I(X_1 ; X_2)$. 
While mutual information measures the strength of a relationship, the ``common information'' provides a concrete representation, $Y$, of the information that is shared between two variables. 
According to \cite{wyner_common}, if $Y$ contains the common information between $X_1, X_2$, then we should have $I(X_1;X_2 |Y)=0$, i.e., $Y$ makes the variables conditionally independent. 
We can extend this idea to many variables using the multivariate generalization of mutual information called total correlation~\cite{watanabe}, so that conditional independence is equivalent to the condition $TC(X_1,\ldots, X_n|Y)=0$~\cite{xu_wyner}. 
The most succinct $Y$ that has this property represents the multivariate common information in $X$ but finding such a $Y$ in general is a challenging, unsolved problem. 

The main contribution of this paper is to show that the concept of common information, long studied in information theory for applications like distributed source coding and cryptography~\cite{common}, is also a useful concept for machine learning.  
Machine learning applications have been overlooked due to the intractability of recovering common information for high-dimensional problems. 
We propose a concrete and tractable algorithmic approach to extracting common information by exploiting a connection with the recently introduced ``information sieve'' decomposition~\cite{sieve}. 
%We extend the information sieve beyond its original formulation for discrete variables to handle continuous variables as well. 
The sieve decomposition works by searching for a single latent factor that reduces the conditional dependence in the data as much as possible. Then the data is transformed to remove this dependence and the ``remainder information'' trickles down to the next layer. The process is repeated until all the dependence has been extracted and the remainder contains nothing but independent noise. Thm.~\ref{decomp} connects the latent factors extracted by the sieve to a measure of common information.

Our second contribution is to show that under the assumptions of linearity and Gaussianity this optimization has a simple fixed-point solution (Eq.~\ref{eq:fixed}) with fast convergence and computational complexity linear in the number of variables. 
Although our final algorithm is limited to the linear case, extracting common information is an unsolved problem and our approach represents a logical first step in exploring the value of common information for machine learning. We offer suggestions for generalizing the method. 

Our final contribution is to validate the usefulness of our approach on some canonical machine learning problems. 
While PCA finds components that explain the most variation, the sieve discovers components that explain the most dependence, making it a useful complement for exploratory data analysis. 
Common information can be used to solve a natural class of blind source separation problems that are impossible to solve using independent component analysis (ICA) due to the presence of Gaussian sources. Finally, we show that common information outperforms standard approaches for dimensionality reduction and recovering structure in fMRI data.

\section{Preliminaries}\label{sec:background}
Using standard notation~\cite{cover}, capital $X_i$ denotes a continuous random variable whose instances are denoted in lowercase, $x_i$. We abbreviate multivariate random variables, $X \equiv X_{1:n} \equiv X_1,\ldots,X_n$, with an associated probability density function, $p_X(X_1=x_1,\ldots, X_n=x_n)$, which is typically abbreviated to $p(\vx)$, with vectors in bold.  We will index different groups of multivariate random variables with superscripts, $X^k$, as defined in Fig.~\ref{fig:sieve}. We let $X^0$ denote the original observed variables and we omit the superscript for readability when no confusion results. 

Entropy is defined as $H(X) \equiv \langle \log 1/p(\vx) \rangle$, where we use brackets for expectation values. Conditional multivariate mutual information, or conditional total correlation, is defined as the Kullback-Leibler divergence between the joint distribution, and the one that is conditionally independent. 
\be\label{eq:tc}
TC(X|Y) \equiv  D_{KL}\left(p(\vx | y) \Big \| \prod_{i=1}^n p(x_i | y)\right)
\ee
This quantity is non-negative and zero if and only if all the $X_i$'s are independent conditioned on $Y.$ 
$TC(X)$ can be obtained by dropping the conditioning on $Y$ in the expression above. In other words, $TC(X)=0$ if and only if the variables are (unconditionally) independent. 
If $Y$ were the hidden source of all dependence in $X$, then $TC(X|Y)=0$. 
Therefore, we consider the problem of searching for a factor $Y$ that minimizes $TC(X|Y)$. 
In the statement of the theorems we make use of shorthand notation, $TC(X;Y) \equiv TC(X)-TC(X|Y)$, which is the reduction of TC after conditioning on $Y$.  This notation mirrors the definition of mutual information between two groups of random variables, $X$ and $Y$, as the reduction of uncertainty in one variable, given information about the other, $ I(X;Y) = H(X) - H(X|Y)$.

\section{Extracting Common Information}\label{sec:decomposition}

For $Y$ to contain the common information in $X$, we need $TC(X|Y)=0$. 
Instead of enforcing the condition that $TC(X|Y)=0$ and looking for the most succinct $Y$ that satisfies this condition, as Wyner does~\cite{wyner_common}, we consider the dual formulation where we minimize $TC(X|Y_1,\ldots,Y_r)$ subject to constraints on $r$, the size of the state space~\cite{gastpar}. 
This optimization can be written equivalently as follows. 
\begin{eqnarray}\label{eq:opt1}
\min_{\vy = \vf(\vx)} TC(X_1,\ldots,X_n|Y_1,\ldots,Y_r)
\end{eqnarray}
We will show in Thm.~\ref{decomp} that an upper bound for this objective is obtained by solving a sequence of optimization problems of the following form, indexed by $k$. 
\begin{eqnarray}\label{eq:opt_decompose}
\min_{y_k = f(\vx^{k-1})} TC(X_1^{k-1},\ldots,X_{n_k}^{k-1}|Y_k)
\end{eqnarray}
The definition of $X^{k}$ is discussed next, but the high level idea is that we have reduced the difficult optimization over many latent factors in Eq.~\ref{eq:opt1} to a sequence of optimizations with a single latent factor in Eq.~\ref{eq:opt_decompose}. Each optimization gives us a tighter upper bound on our original objective, Eq.~\ref{eq:opt1}.

\paragraph{Incremental Decomposition} We begin with some input data, $X$, and then construct $Y_1$ to minimize $TC(X|Y_1)$.  After doing so, we would like to transform the original data into the remainder information, $X^1$, so that we can use the same optimization to learn a factor, $Y_2$, that extracts more common information that was not already captured by $Y_1$. We diagram this construction at layer $k$ in Fig.~\ref{fig:sieve} and show in Thm~\ref{incremental} the requirements for constructing the remainder information. 
The result of this procedure is encapsulated in Cor.~\ref{iterative} which says that we can iterate this procedure and $TC(X|Y_1,\ldots,Y_k)$ will be reduced at each layer until it reaches zero and $Y$ captures all the common information.
%$TC(X)$ decomposes into a sum of contributions from each $Y_k$. 
\begin{figure}[tbp] %  figure placement: here, top, bottom, or page
   (a)\includegraphics[width=0.93\columnwidth]{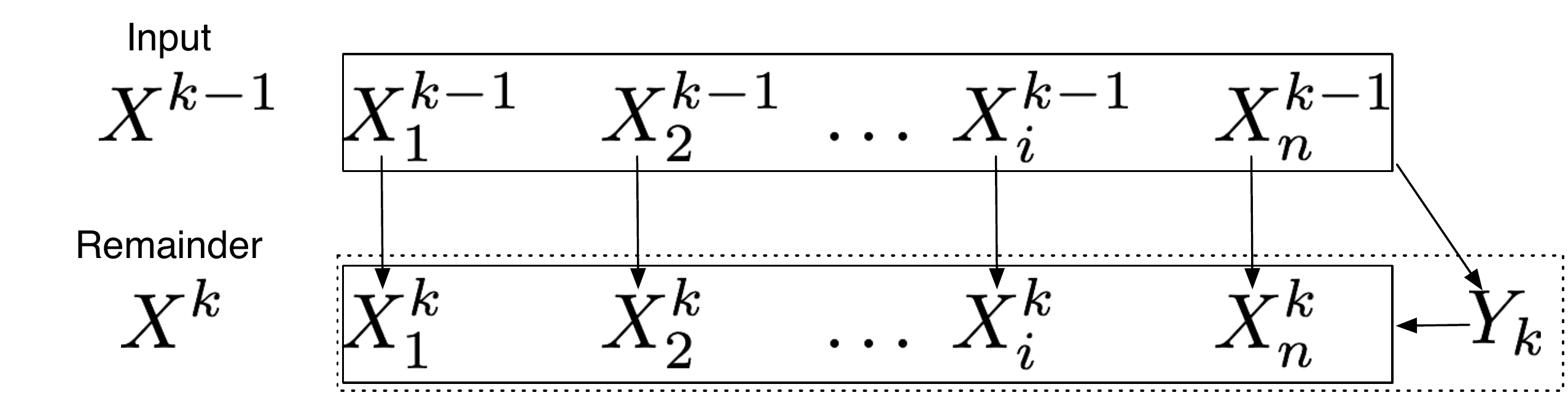} \\ %hierarchical_decomp.pdf}
   \begin{center}
   (b)\begin{minipage}[c]{0.9\columnwidth}
$$
\arraycolsep=1.6pt\def\arraystretch{1.2}
   \begin{array}{rcccccc} % brackets may be (...), [...], \{...\}, or left out
      X^0: & \bf X_1 & \ldots & \bf X_n & & & \\
      X^1: & X_1^1 & \ldots & X_n^1 & \color{red} Y_1 & & \\
      X^2: & X_1^2 & \ldots & X_n^2 &Y_1^2 & \color{red} Y_2 & \\
      \cdots \\
      X^k: & X_1^k & \ldots & X_n^k &Y_1^k &Y_2^k & \color{red} Y_k \\
   \end{array}
$$
\end{minipage}
\end{center}
   \caption{\small (a) This diagram describes one layer of the sieve. $Y_k$ is some function of the $X_i^{k-1}$'s that is optimized to capture dependence. The remainder, $X_i^k$ contains information that is not explained by $Y_k$. (b) We summarize the naming convention for multiple layers. 
   }  % \vspace{-2mm}
   \label{fig:sieve}
\end{figure}

\begin{theorem} \label{incremental}
{\bf Incremental decomposition of common information}\quad
\label{exact}
For $Y_k$ a function of $X^{k-1}$, the following decomposition holds,
\be\label{eq:exact}
TC(X^{k-1}) = TC(X^k) + TC(X^{k-1};Y_k), 
\ee
if the remainder information $X^k$ satisfies two properties. \\
\-\quad 1. Invertibility: there exist functions $g, h$ so that \\ \-\qquad $x_i^{k-1} = g(x_i^k, y_k)$ and  $x_i^k = h(x_i^{k-1}, y_k)$\\
\-\quad 2. Remainder contains no information about $Y_k$: \\ \-\qquad $\forall i, I(X_i^k ; Y_k) = 0$
\end{theorem} 
\begin{proof}
We refer to Fig.~\ref{fig:sieve}(a) for the structure of the graphical model. We set $\bar X \equiv \bar X_1,\ldots, \bar X_n, Y$ and we will write $\bar X_{1:n}$ to pick out all terms except $Y$. 
Expanding the definition of $TC(X;Y)$, the equality in Eq.~\ref{eq:exact} becomes
\benn
TC(\bar X) - TC(X|Y) = \left\langle \log \frac{p(\bar x, y) \prod_i p(x_i|y)}{p(y) p(x|y) \prod_i p(\bar x_i) } \right\rangle = 0
 \eenn 
We have to show that this quantity equals zero under the assumptions specified. First, we multiply the fraction by one by putting $\prod_i p(\bar x_i |y)$ terms in the numerator and denominator.  
After applying condition (2) that $I(\bar X_i;Y)=0$, we can remove two terms leaving the following. 
\benn
\left\langle \log \frac{p(\bar x | y) \prod_i p(x_i|y)}{p(x | y) \prod_i p(\bar x_i|y) } \right\rangle
\eenn
If condition (1) of the theorem is satisfied, then, conditioned on $y$, $\bar x_i$ and $x_i$ are related by a deterministic formula. We can see from applying the change of variables formula for probability distributions that the terms in this expression cancel, leaving us with $\la \log 1 \ra=0$, as we intended to prove. 
\end{proof}
The decomposition above was originally introduced for discrete variables as the ``information sieve''~\cite{sieve}; the continuous formulation we introduce here replaces the first condition used in the original statement with an analogous one that is appropriate for continuous variables.
%and the proof for this new version is in Sec.~\ref{sec:proof}.
Note that because we can always find non-negative solutions for $TC(X^{k-1};Y_k)$, it must be that $TC(X^k) \leq TC(X^{k-1})$. In other words, the remainder information is more independent than the input data. This is consistent with the intuition that the sieve is sifting out the common information at each layer.
%
% ITERATIVE COROLLARY
\begin{corollary}{\bf Iterative decomposition of TC} \quad
\label{iterative}
With a hierarchical representation where each $Y_k$ is a function of $X^{k-1}$ and $X^k$ is the remainder information as defined in Thm~\ref{incremental}, $
TC(X) = TC(X^r) +  \sum_{k=1}^r TC(X^{k-1} ; Y_k).$
\end{corollary} % END COROLLARY
This follows from repeated application of Eq.~\ref{eq:exact}. 
%As we add the (non-negative) contributions from optimizing $TC(X^{k-1} ; Y_k)$, the dependence in the remainder information, $TC(X^k)$, must decrease. Therefore, we consider comparisons to independent component analysis in Sec.~\ref{sec:results}. Finally, this suggests a natural stopping procedure for this learning framework; we should add a new latent factor $Y_k$ only as long as we are getting a significant contribution from $TC(X^{k-1} ; Y_k)$.  
$TC(X)$ is a constant that depends on the data. For high-dimensional data, it is impossible to measure $TC(X)$, but by learning latent factors extracting progressively more dependence, we get a sequence of better bounds.  

\begin{theorem} \label{decomp}
{\bf Decomposition of common information}\quad For the sieve decomposition, the following bound holds.
\benn
TC(X|Y_{1:r}) & \leq TC(X^r) = TC(X) -  \sum_{k=1}^r TC(X^{k-1} ; Y_k)
\eenn
\end{theorem}
\begin{proof}
The equality comes from Cor.~\ref{iterative}. 
\benn
\lefteqn{TC(X_{1:n}|Y_{1:r})}\\
 &= \left\langle \log \frac{ p(x_{1:n}|y_{1:r})}{\prod_{i=1}^n p(x_i|y_{1:r})} \right\rangle = \left\langle \log \frac{p(x_{1:n}^r | y_{1:r}^r )}{\prod_{i=1}^n p(x_i^r | y_{1:r}^r )} \right\rangle \\
&= \left\langle \log \frac{p(x_{1:n}^r, y_{1:r}^r )}{\prod_{i=1}^n p(x_i^r) \prod_{k=1}^r p(y_k^r)} 
\frac{\prod_{i=1}^n p(x_i^r) \prod_{k=1}^r p(y_k^r)}{p(y_{1:r}^r) \prod_{i=1}^n p(x_i^r | y_{1:r}^r )} \right\rangle \\
&= TC(X^r) + \left\langle \log \frac{\prod_{i=1}^n p(x_i^r) \prod_{k=1}^r p(y_k^r)}{p(y_{1:r}^r) \prod_{i=1}^n p(x_i^r | y_{1:r}^r )} \right\rangle \\
&= TC(X^r) - TC(Y^r_{1:r}) - \sum_{i=1}^n I(X_i^r;Y^r_{1:r}) \\
& \leq TC(X^r)
\eenn
The first line follows from the the change of variables formula for the transformation connecting layer $r$ to the input layer. On the second line we multiply by 1 and re-arrange, collecting terms in the next two lines. 
The last inequality follows from non-negativity of TC and mutual information. 
\end{proof}

Recalling that $TC(X;Y) = TC(X)-TC(X|Y)$, Thm.~\ref{decomp} shows how the sum of terms optimized in Eq.~\ref{eq:opt_decompose} provide a successively tighter upper bound on the objective of Eq.~\ref{eq:opt1}. 
% is proved in Sec.~\ref{sec:ci} and 
In other words, as we keep adding and optimizing latent factors they reduce the conditional TC until all the common information has been extracted.
%When $TC(X^r)=0$ then $TC(X|Y_1,\ldots,Y_r)=0$, i.e., $Y$ has captured all the common information in $X$. 

\paragraph{Optimization} It remains to solve the optimization in Eq.~\ref{eq:opt_decompose}. For now we drop the $k$ index and focus on minimizing $TC(X|Y)$ for a single factor $Y$. To get a simple and tractable solution to this non-convex problem, we consider a further simplification where $X$ is Gaussian with covariance matrix $\Sigma$ and inverse covariance $\Lambda=\Sigma^{-1}$.  If $X$ is Gaussian and $Y$'s dependence on $X$ is linear and Gaussian, the joint distribution over $X, Y$ will also be Gaussian. We write out the optimization in Eq.~\ref{eq:opt_decompose} under this condition. 
\be\label{eq:opt2}
\min_{Y|X \sim \mathcal N(\vw \cdot \vx, \eta^2)} \sum_{i=1}^n H(X_i|Y) - H(X|Y)
\ee
Two immediate simplifications are apparent. First, this objective is invariant to scaling of $Y$. Any solution with $\eta, \vw$ would be equivalent to a scaled solution $s \eta, s \vw$. Therefore, without loss of generality we set $\eta=1$. Second, we invoke Bayes rule to see $H(X|Y) = H(Y|X) + H(X) - H(Y)$ where the first two terms on the right hand side are constants with respect to the optimization. We re-write the optimization accordingly. 
\benn
\min_{Y|X \sim \mathcal N(\vw \cdot \vx, 1)} \sum_{i=1}^n H(X_i|Y) + H(Y)
\eenn
The objective is invariant to translation of the marginals, so w.l.o.g. we also set $\langle X_i \rangle = \langle Y \rangle =0$. 
Define a nonlinear change of variables in terms of the correlation coefficient, $\rho_i = \la X_i Y\ra/\sqrt{\la X_i^2 \ra \la Y^2 \ra}$. To translate between $\vw$ and $\vr$, we also note, $(\Sigma \vw)_i = \la X_i Y\ra, \vw = \Lambda \vr \sqrt{\la X_i^2 \ra \la Y^2 \ra}$ and $\la Y^2 \ra = 1 / (1 -\vr^\top \Lambda \vr) = \vw^\top \Sigma \vw + 1$. 
This leads to the following optimization, neglecting some constants.  
\benn
\min_{Y|X \sim \mathcal N(\vw \cdot \vx, 1)} \sum_{i=1}^n 1/2 \log(1 - \rho_i^2) - 1/2 \log (1-\rho^\top \Lambda \rho)
\eenn
Next, we set derivatives with respect to each $\rho_i$ to zero. 
$$\partial_{\rho_i} \mathcal TC(X|Y) = -\rho_i / (1-\rho_i)^2 + \Lambda \vr /(1-\vr^\top \Lambda \vr)  = 0.$$ 
Now we use the identities to translate back to a fixed-point equation in terms of $\vw$ and rearrange. 
\begin{eqnarray}\label{eq:fixed}
w_i = \frac{\langle X_i Y \rangle}{\langle X_i^2 \rangle  \langle Y^2 \rangle - \langle X_i Y \rangle^2}
\end{eqnarray}
Interestingly, we arrive at a novel nonlinear twist on the classic Hebbian learning rule~\cite{baldi}. If $X_i$ and $Y$ ``fire together they wire together'' (i.e. correlations lead to stronger weights), but this objective strongly prefers correlations that are nearly maximal, in which case the denominator becomes small and the weight becomes large. 
This optimization of $TC(X|Y)$ for continuous random variables $X$ and $Y$ is, to the best of our knowledge, the first tractable approach except for a special case discussed by \cite{gastpar}. Also note that although we used $\Sigma, \Lambda$ in the derivation, the solution does not require us to calculate these computationally intensive quantities.

A final consideration is the construction of remainder information (i.e., how to get $X^k$ from $X^{k-1}$ and $Y$ in Fig.~\ref{fig:sieve}) consistent with the requirements in Thm.~\ref{exact}. In the discrete formulation of the sieve, constructing remainder information is a major problem that ultimately imposes a bottleneck on its usefulness because the state space of remainder information can grow quickly. In the linear case, however, the construction of remainder information is a simple linear transformation reminiscent of incremental PCA. 
We define the remainder information with a linear transformation, $X_i^k = X_i^{k-1} - \langle X_i^{k-1} Y_k \rangle / \langle Y_k^2 \rangle  Y_k$. This transformation is clearly invertible (condition (i)), and it can be checked that $\langle X_i^k Y_k \rangle = 0$ which implies $I(X_i^k ;Y_k)=0$ (condition (2)). 

\paragraph{Generalizing to the Non-Gaussian, Nonlinear Case}
The solution for the linear, Gaussian case is more flexible than it looks. We do not actually have to require that the data, $X$, is drawn from a \emph{jointly} normal distribution to get meaningful results. It turns out that if each of the individual marginals is Gaussian, then the expression for mutual information for Gaussians provides a lower bound for mutual information~\cite{mi_bound}. Also, the objective (Eq.~\ref{eq:opt1}) is invariant under invertible transformations of the marginals~\cite{cover}. Therefore, to ensure that the optimization that we solved (Eq.~\ref{eq:opt2}) is a lower bound for the optimization of interest, we should transform the marginals to be individually Gaussian distributed.  
Several nonlinear, parametric methods to Gaussianize one-dimensional data exist, including a recent method that works well for long-tailed data~\cite{lambert}. Alternatively, a nonparametric approach is to Gaussianize data based on the rank statistics~\cite{order_tests}. Finally, \cite{nonparanormal} study information measures for a large family of distributions that can be nonparametrically transformed into normal distributions.
%The basic idea is to replace each one-dimensional sample with its rank and then use the inverse CDF of the Gaussian to transform it into data distributed normally. 

%A problem remains; we argued that as long as the marginals are Gaussian, we can use the Gaussian MI as a lower bound for the true MI. However, if $Y$ is a linear function of $X$ and the marginals, $X_i$, are (individually but not jointly) Gaussian, then this does not guarantee that $Y$ is also Gaussian. Generally, we ignore this issue and take Eq.~\ref{eq:opt2} as our optimization problem. However, in principle, one could introduce a Gaussianizing nonlinearity for $Y$ as part of the optimization. Alternately, we could follow ICA~\cite{ica} and add a term to the optimization that models the non-Gaussianity of $Y$. Despite ignoring this key ingredient of ICA, we show in Sec.~\ref{sec:results} that the sieve outperforms ICA for some blind source separation problems. 
%\footnote{This is a notable deviation from ICA where modeling the non-Gaussianity of $Y$ is the main technical focus~\cite{ica}.}

\section{Implementation Details}\label{sec:implement}

\paragraph{A Single Layer} A concrete implementation of one layer of the sieve transformation is straightforward and the algorithm is summarized in Alg.~\ref{alg1}. Our implementation is available online~\cite{code_sieve}. The minimal preprocessing of the data is to subtract the mean of each variable. Optionally, further Gaussianizing preprocessing can be applied. Our fixed point optimization requires us to start with some weights, $\vw^{0}$ and we iteratively update $\vw^{t}$ using Eq.~\ref{eq:fixed} until we reach a fixed point. This only guarantees that we find a local optima so we typically run the optimization 10 times and take the solution with the highest value of the objective. We initialize $\vw^0_i$ to be drawn from a normal with zero mean and scale $1/\sqrt{n \sigma_{x_i}^2}$. 
We scale each $w_i^0$ by the standard deviation of each marginal so that one variable does not strongly dominate the random initialization, $y = \vw^0 \cdot \vx$. 
%While the scale, $\eta^2$, is arbitrary, it should be set so that the range of values for $y, \vw$ do not lose numerical precision.

\begin{algorithm}
{ \small
 \KwData{Data matrix, $N$ iid samples of vectors, $\vx \in \mathbb R^n$}
 \KwResult{Weights, $\vw$, so that $y=\vw \cdot \vx$ optimizes $TC(X;Y)$ and remainder information, $\bar \vx$.}
 Subtract mean from each column of data\;
 Initialize $w_i \sim \mathcal N(0, 1 / (\sqrt{n} \sigma_{x_i}))$\;
 \While{not converged}{
  Calculate $y=\vw \cdot \vx$ for each sample \;
  Calculate moments from data, $\langle X_i Y \rangle, \langle Y^2 \rangle, \langle X_i^2 \rangle$\;
  $\forall i, w_i \leftarrow  \langle X_i Y \rangle / (\langle Y^2 \rangle \langle X_i^2 \rangle - \langle X_i Y \rangle^2) $\;
  }
  For each column of data, $i$, return $\bar x_i = x_i - \frac{\langle X_i Y\rangle}{\langle Y^2\rangle} y$ \;
  }
   \caption{\small Algorithm to learn one layer of the sieve.}\label{alg1}
\end{algorithm}

The iteration proceeds by estimating marginals and then applying Eq.~\ref{eq:fixed}. 
%We do not actually add noise when calculating $y$, because $\langle X_i Y \rangle$ does not depend on the noise and we can analytically correct the variance of $Y$, $\langle Y^2 \rangle = \langle (\sum_i w_i X_i)^2 \rangle + \eta^2$.  
Estimating the covariance at each step is the main computational burden, but the steps are all linear. If we have $N$ samples and $n$ variables, then we calculate labels for each data point, $y = \vw \cdot \vx$, which amounts to $N$ dot products of vectors with length $n$. Then we calculate the covariance, $\langle X_i Y \rangle$, which amounts to $n$ dot products of vectors of length $N$. These are the most intensive steps and could be easily sped up using GPUs or mini-batches if $N$ is large. Convergence is determined by checking when changes in the objective of Eq.~\ref{eq:opt2} fall below a certain threshold, $10^{-8}$ in our experiments.  

%At the end of training one layer of the sieve, we should also output the remainder information for use in subsequent layers. By definition, the latent factor $Y$, includes some random independent noise, $y = \vw \cdot \vx + \epsilon$. In training the next layer, we include this random noise to increase robustness. However, in testing, we do not need to include this extra source of noise.   
%In principle, we could use a more complex nonlinear (but invertible) function to ensure that $I(X_i^k;Y_k)=0$ is met in cases where the data is not Gaussian, but we leave this to future work. 

\paragraph{Multiple Layers} 
After training one layer of the sieve, it is trivial to take the remainder information and feed it again through Alg.~\ref{alg1}. While our optimization in Eq.~\ref{eq:opt2} formally involved a probabilistic function, we take the final learned function to be deterministic, $y = \vw \cdot \vx$, as required by Thm.~\ref{incremental}.  
Each layer contributes $TC(X^{k-1};Y_k)$ in our decomposition of $TC(X)$, so we can stop when these contributions become negligible. This occurs when the variables in $X^k$ become independent. In that case, $TC(X|Y_{1:k}) = TC(X^k)=0$ and since $TC(X^k) \geq TC(X^{k};Y_{k+1})$, we get no more positive contributions from optimizing $TC(X^{k};Y_{k+1})$. 
%In practice, for finite samples of independent data, we might find non-zero TC due to biased estimation~\cite{gaussian_bias}.  

\section{Results}\label{sec:results}

We begin with some benchmark results on a synthetic model. We use this model to show that the sieve can uniquely recover the hidden sources, while other methods fail to do so. 
%Finally, we consider the efficacy of using the sieve for dimensionality reduction on real-world data sets. 

\begin{SCfigure}[][tbp]
   \includegraphics[width=0.5\columnwidth]{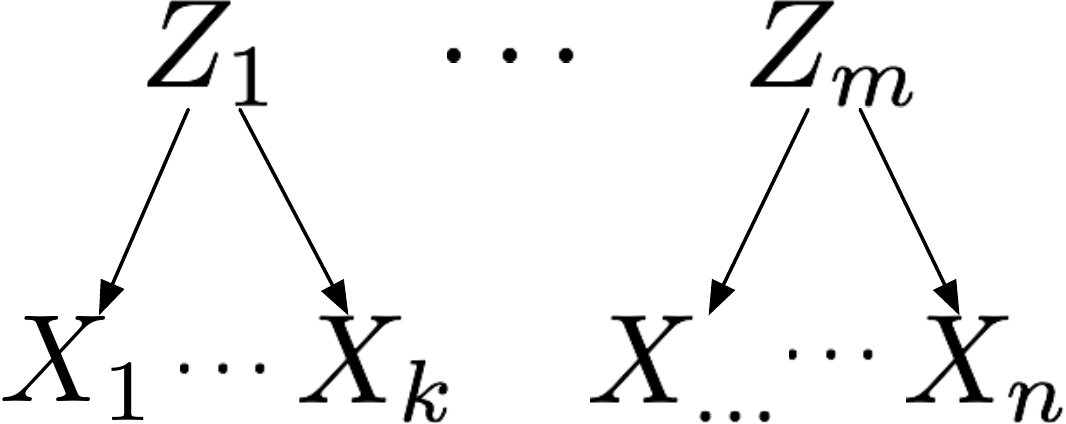}
   \caption{This is the generative model used for synthetic experiments. Each observed variable, $X_i = Z_{pa(i)} + \varepsilon_i$ combines its parent, $Z_{pa(i)}$, with Gaussian noise.}\label{fig:model}
\end{SCfigure}

\paragraph{Data Generating Model}
For the synthetic examples, we consider data generated according to a model defined in Fig.~\ref{fig:model}. We have $m$ sources, each with unit variance, $Z_j \sim \mathcal N (0, 1)$. Each source has $k$ children and the children are not overlapping. Each channel is an additive white Gaussian noise (AWGN) channel defined as $X_i = Z_{pa(i)} + \varepsilon_i$. The noise has some variance that may be different for each observed variable, $\varepsilon_i \sim \mathcal N (0, \epsilon_i^2)$.  Each channel can be characterized as having a capacity, $C_i = 1/2 \log (1 + 1/\epsilon_i^2)$~\cite{cover}, and we define the total capacity, $C = \sum_{i=1}^k C_i$. For experiments, we set $C$ to be some constant, and we set the noise so that the fraction, $C_i / C$, allocated to each variable, $X_i$, is drawn from the uniform distribution over the simplex. 

\paragraph{Empirical Convergence Rates} 
We examine how quickly the objective converges by plotting the error at the $t$-th iteration. 
The error is defined as the difference between TC at each iteration and the final TC. We take the final value of TC to be the value obtained when the magnitude of changes falls below $10^{-14}$. We set $C=1$ for these experiments. 
In Fig.~\ref{fig:convergence}, we look at convergence for a few different settings of the generative model and see linear rates of convergence (where error is plotted on a log scale, as is conventional for convergence plots), with a coefficient that seems to depend on problem details. The slowest rate of convergence comes from data where each $X_i$ is generated from an independent normal distribution (i.e., there is no common information).
\begin{figure}[tbp] %  figure placement: here, top, bottom, or page
   \centering
   \includegraphics[width=0.9\columnwidth]{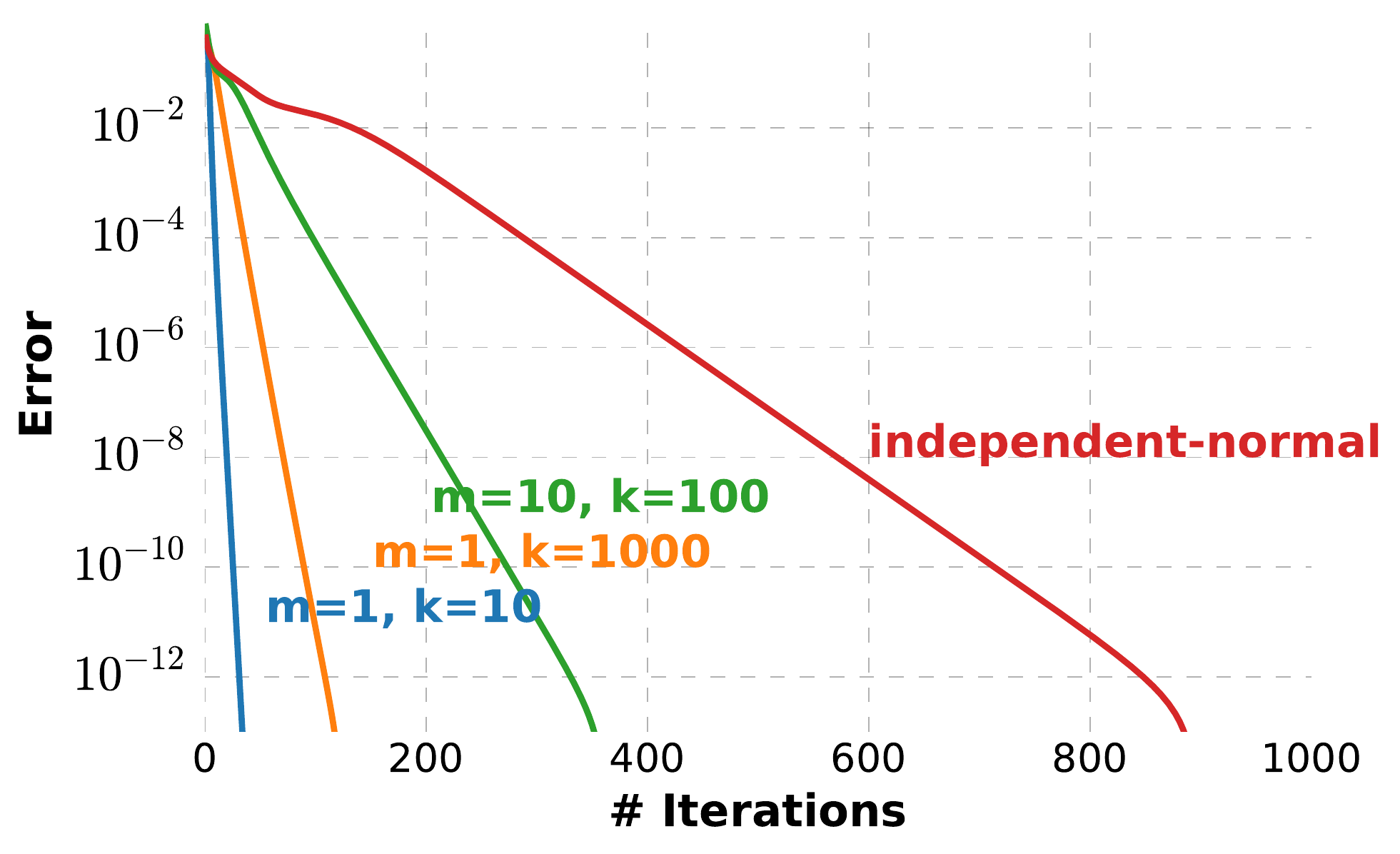} \qquad
     %(b) \includegraphics[width=0.95\columnwidth]{Figs/convergence2.pdf} 
   \caption{Empirical error plots on synthetic data show linear rates of convergence. We obtained similar results on real-world data. }
   \label{fig:convergence}
\end{figure}

\paragraph{Recover a Single Source from Common Information}
As a first test of performance, we consider a simple version of the model in Fig.~\ref{fig:model} in which we have just a single source and we have $k$ observed variables that are noisy copies of the source. For this experiment, we set total capacity to $C=4$. 
%Intuitively, this reflects a Shannon upper bound that says that we should not be able to recover more than about two decimal bits from our source. 
By varying $k$, we are spreading this capacity across a larger number of noisier variables. We use the sieve to recover a single latent factor, $Y$, that captures as much of the dependence as possible (Eq.~\ref{eq:opt1}), and then we test how close this factor is to the true source, $Z$, using Pearson correlation. We also compare to various other standard methods: PCA~\cite{halko2011finding}, ICA~\cite{ica}, Non-Negative Matrix Factorization (NMF)~\cite{lin2007projected}, Factor Analysis (FA)~\cite{cattell}, Local Linear Embedding (LLE)~\cite{roweis2000nonlinear}, Isomap~\cite{tenenbaum2000global}, Restricted Boltzmann Machines (RBMs)~\cite{hintonRBM}, and k-Means~\cite{sculley2010web}.
All methods were run using implementations in the scikit library~\cite{pedregosa2011scikit}.

\begin{figure*}[tbp] %  figure placement: here, top, bottom, or page
   \centering
(a)   \includegraphics[width=0.9\columnwidth]{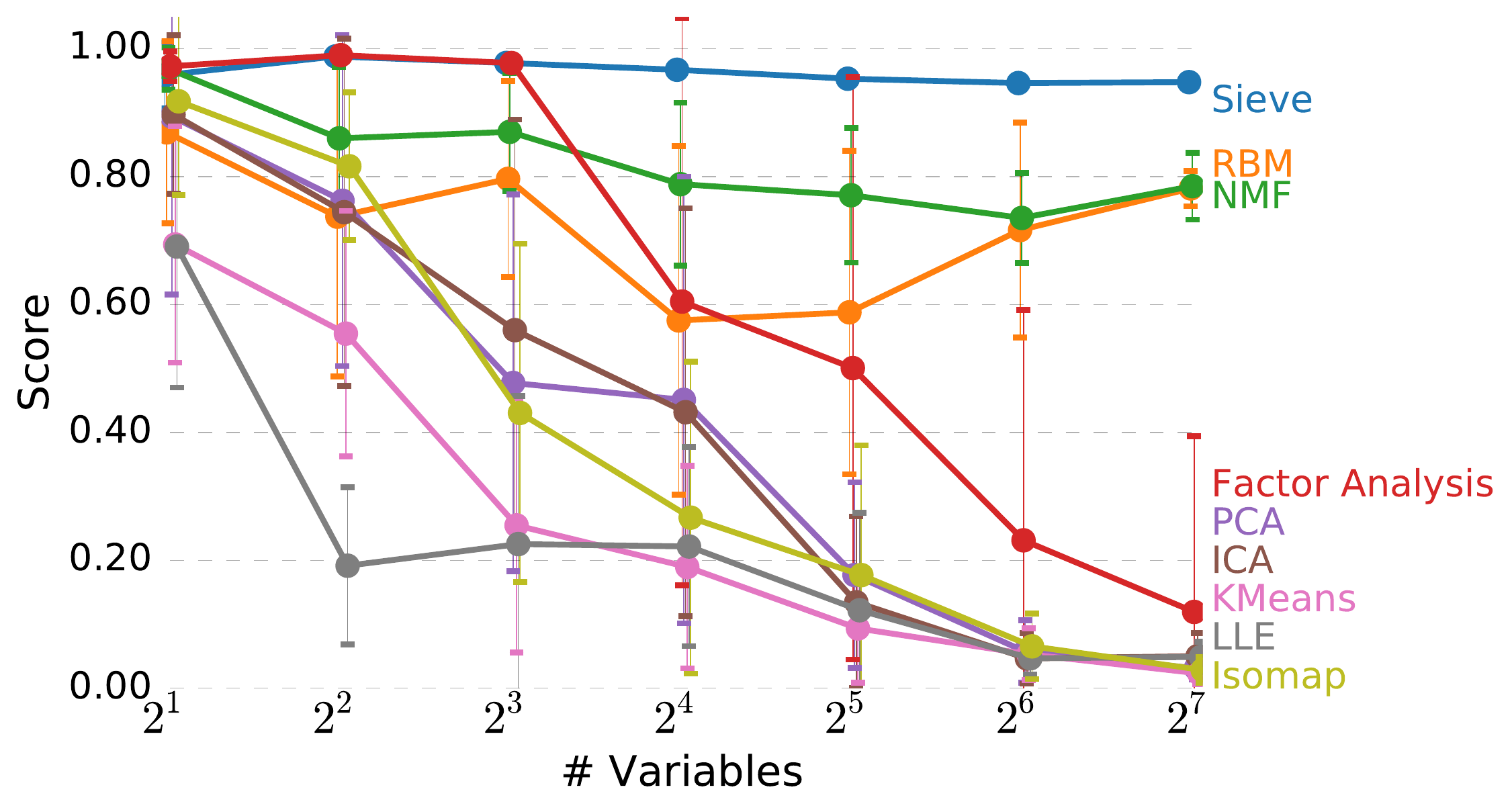} 
   (b)    \includegraphics[width=0.9\columnwidth]{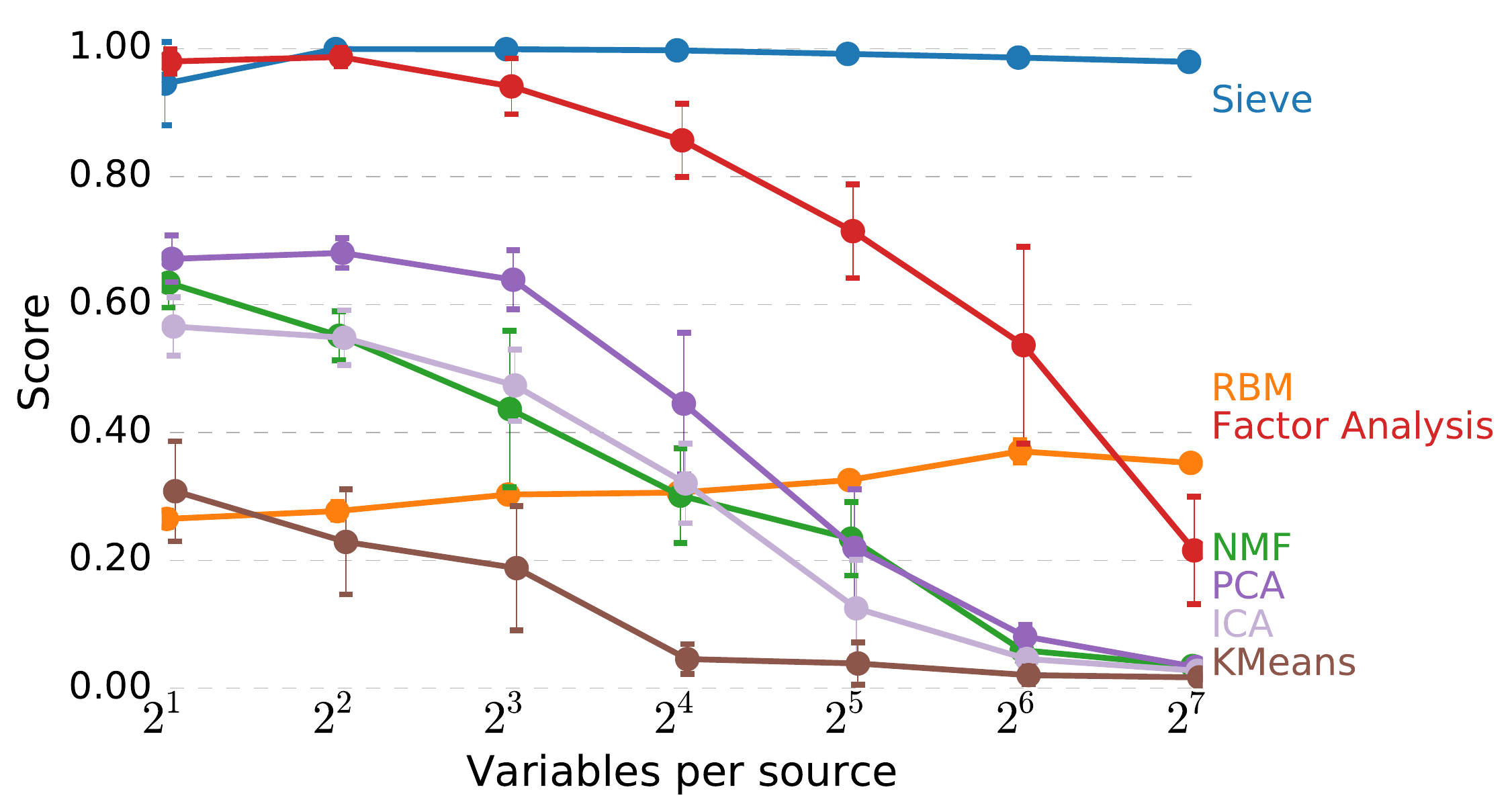} 
   \caption{Each source is compared to the best match of the components returned by each method. The score is the average of the absolute Pearson correlations. Each point is a mean score over ten randomly generated datasets, with error bars representing standard deviation. 
   (a) We attempt to recover a single hidden source variable from varying numbers of observed variables. We set $C=4$ and use 500 samples. 
   (b) We attempt blind source separation for ten independent, hidden source variables given varying numbers of observed variables per source.
We set $C=12$ and use 10000 samples.  } \vspace{-2mm}
   \label{fig:single}
\end{figure*}

Looking at the results in Fig.~\ref{fig:single}(a), we see that for a small number of variables almost any technique suffices to recover the source. As the number of variables rises, however, intuitively reasonable methods fail and only the sieve maintains high performance. The first component of PCA, for instance, is the projection with the largest variance but it can be shown that by changing the scale of the noise in different directions, this component can be made to point in any direction.
%\footnote{To some extent this can be characterized analytically by noting the first PCA component is the largest eigenvector of the covariance matrix, $\Sigma_X$. In this case, $\Sigma_X$ is a diagonal matrix with a rank-one perturbation and the eigenvectors of this type of system have been well-studied~\cite{rankone}.} 
Unlike PCA, the sieve is invariant under scale transformations of each variable. 
Error bars are produced by looking at the standard deviation of results over 10 randomly generated datasets. Some error bars are smaller than the plot markers. Besides being the most accurate method, the sieve also has the smallest variance.

\subsection{Source Separation with Common Information}
In the generative model in Fig.~\ref{fig:model}, we have $m$ independent sources that are each Gaussian distributed. We could imagine applying an orthonormal rotation, $R$, to the vector of sources and call these $\tilde Z_j = \sum_k R_{jk} Z_k$. Because of the Gaussianity of the original sources, $\tilde Z$ also represent $m$ independent Gaussian sources. We can write down an equivalent generative model for the $X_i$'s, but each $X_i$ now depends on all the $\tilde Z$ (i.e., $X_i = \sum_j R^{-1}_{i,j} \tilde Z_j + \varepsilon_i$). 
From a generative model perspective, our original model is unidentifiable and therefore independent component analysis cannot recover it~\cite{ica}. 
On the other hand, the original generating model is special because the common information about the $X_i$'s are localized in invidivual sources, while in the rotated model, you need to combine information from all the sources to predict any individual $X_i$. The sieve is able to uniquely recover the true sources because they represent the optimal way to sift out common information. 
%Clearly, the former is a much simpler model that is also easier to interpret since each observation is tied to a single latent factor. 
%Global transformations that relate two high-dimensional vectors like $\tilde Z$ and $X$, on the other hand, are difficult to interpret.

%The existence of this interpretable solution is only possible if we consider that the original features, $X_1, \ldots, X_n$, were meaningful in the first place. If we arbitrarily rotate the observed vectors, then this solution will be lost in the new basis. ICA and PCA, for instance, do not assign any meaning to the original basis; an arbitrary rotation will not affect the results. We refer to this assumption as the ``meaningful basis hypothesis''. While the sieve is invariant under invertible transformations of the marginals, global transformations may lead to different results. The sieve can exploit the structure of dependence in the observed basis to recover solutions that would be impossible to recover in rotation-invariant schemes like PCA or ICA. 

To measure our ability to recover the independent sources in our model, we consider a model with $m=10$ sources and varying numbers of noisy observations. The results are shown in Fig.~\ref{fig:single}(b). 
We learn 10 layers of the sieve and check how well $Y_1,\ldots, Y_{10}$ recover the true sources. We also specify 10 components for the other methods shown for comparison. 
As predicted, ICA does not recover the independent sources. 
While the generative model is in the class described by Factor Analysis (FA), there are many FA models that are equally good generative models of the data. In other words, FA suffers from an identifiability problem that makes it impossible to uniquely pick out the correct model~\cite{cosma_book}. 
In contrast, common information provides a simple and effective principle for uniquely identifying the true sources. 
%Why does the sieve recover the special ``interpretable'' model so well? The sieve is optimized so that each latent factor is individually as informative as possible and the equivalent model in Fig.~\ref{fig:model}(b) does not have this property. 

%Another surprising aspect of this analysis is that all the other techniques learn 10 components in parallel, while the sieve can only learn one component at a time, seemingly a disadvantage. In future work, it would be interesting to learn sieve-like representations that optimize the representation over multiple layers simultaneously. While we specified the correct number of components for all methods including the sieve, in principle we could have determined this automatically for the sieve. For $16$ variables per source, for instance, the TC at each layer drops from $9$ nats for the first ten components to $0.1$ for the eleventh component. 

  \begin{figure*}[tbp] %  figure placement: here, top, bottom, or page
   \centering
      \includegraphics[width=0.49\textwidth]{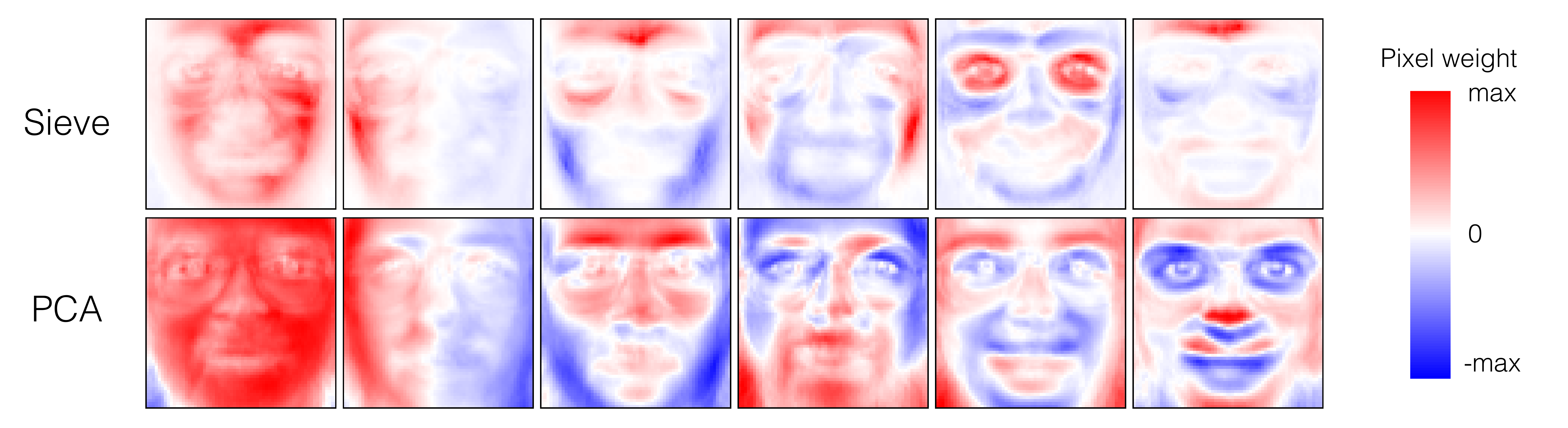}
   \includegraphics[width=0.49\textwidth]{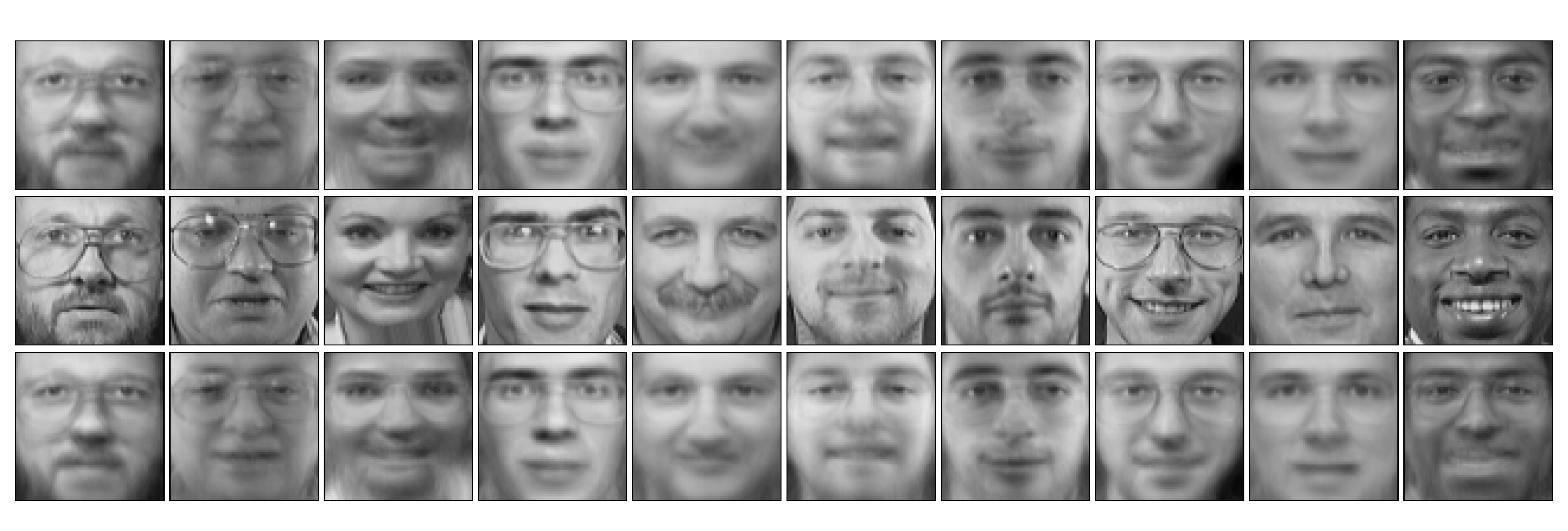} 
   \caption{(Left) The top 6 components for the Olivetti faces dataset using the information sieve (top) and PCA (bottom). Red and blue correspond to negative and positive weights respectively. (Right) We take Olivetti faces (middle row) and then try to reconstruct them using the top 20 components from the sieve (top row) or PCA (bottom row).}
   \label{fig:reconstruct}
\end{figure*}
\paragraph{Exploratory Data Analysis} The first component of PCA explains the most variance in the data, and the weights of the first component are often used in exploratory analysis to understand the semantics of discovered factors. Analogously, the first component of the sieve extracts the largest source of common information.  In Fig.~\ref{fig:reconstruct} we compare the top components learned by the sieve on the Olivetti faces dataset to those learned by PCA. 
The sieve may be more practical for extracting components if data is high dimensional since its complexity is linear in the number of variables while PCA is quadratic.
Like PCA, we can also use the sieve for reconstructing data from a small number of learned factors. Note that the sieve transform is invertible so that $X_i = X_i^1 + \langle X_i^0 Y_1 \rangle / \langle Y_1^2 \rangle Y_1$. If we have a sieve transformation with $r$ layers, then we can continue this expansion as follows.  
$$X_i = X_i^r + \sum_{k=1}^r\langle X_i^{k+1} Y_k \rangle / \langle Y_k^2 \rangle Y_k$$
If we knew the remainder information, $X_i^r$, this reconstruction would be perfect. However, we can simply set the $X_i^r=0$ and we will get a prediction for $X_i$ based only on the learned factors, $Y$, as in Fig.~\ref{fig:reconstruct}.

\paragraph{Source Separation in fMRI Data}
To demonstrate that our approach is practical for blind source separation in a more realistic scenario, we applied the sieve to recover spatial brain components from fMRI data. This data is generated according to a synthetic but biologically motivated model that incorporates realistic spatial modes and heterogeneous temporal signals~\cite{fmri_sim}. We show in Fig.~\ref{fig:mri}(b) that we recover components that match well with the true spatial components. For comparison, we show ICA's performance in Fig.~\ref{fig:mri}(c) which looks qualitatively worse. ICA's poor performance for recovering spatial MRI components is known and various extensions have been proposed to remedy this~\cite{fmri_subject}. This preliminary result suggests that the concept of ``common information'' may be a more useful starting point than ``independent components'' as an underlying principle for brain imaging analysis.
\begin{figure}[tbp] %  figure placement: here, top, bottom, or page
   \centering
   \small
   \includegraphics[width=0.31\columnwidth]{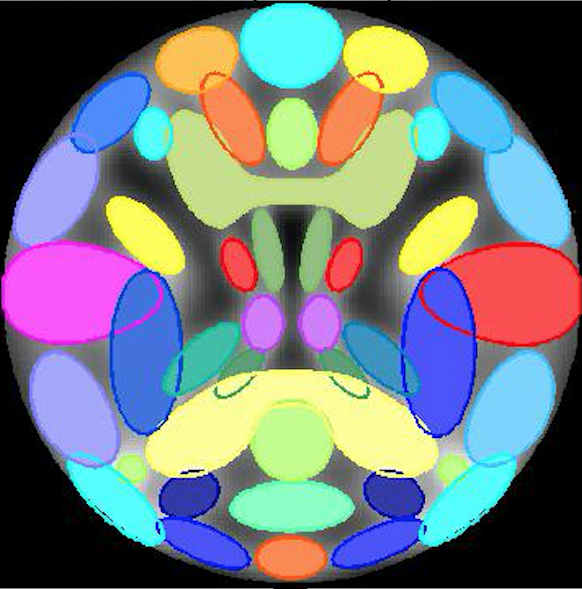}  ~ \includegraphics[width=0.31\columnwidth]{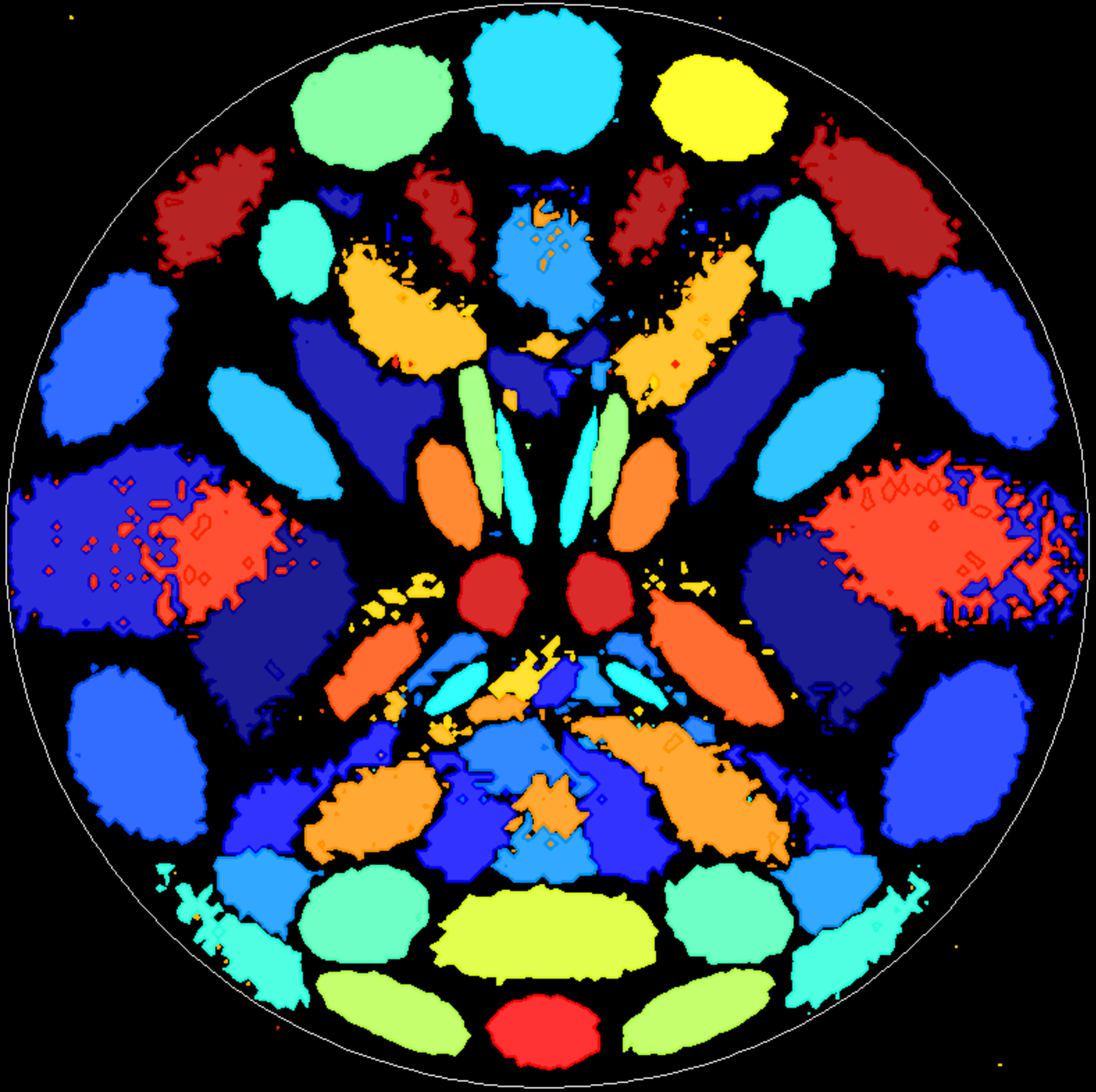} ~  \includegraphics[width=0.31\columnwidth]{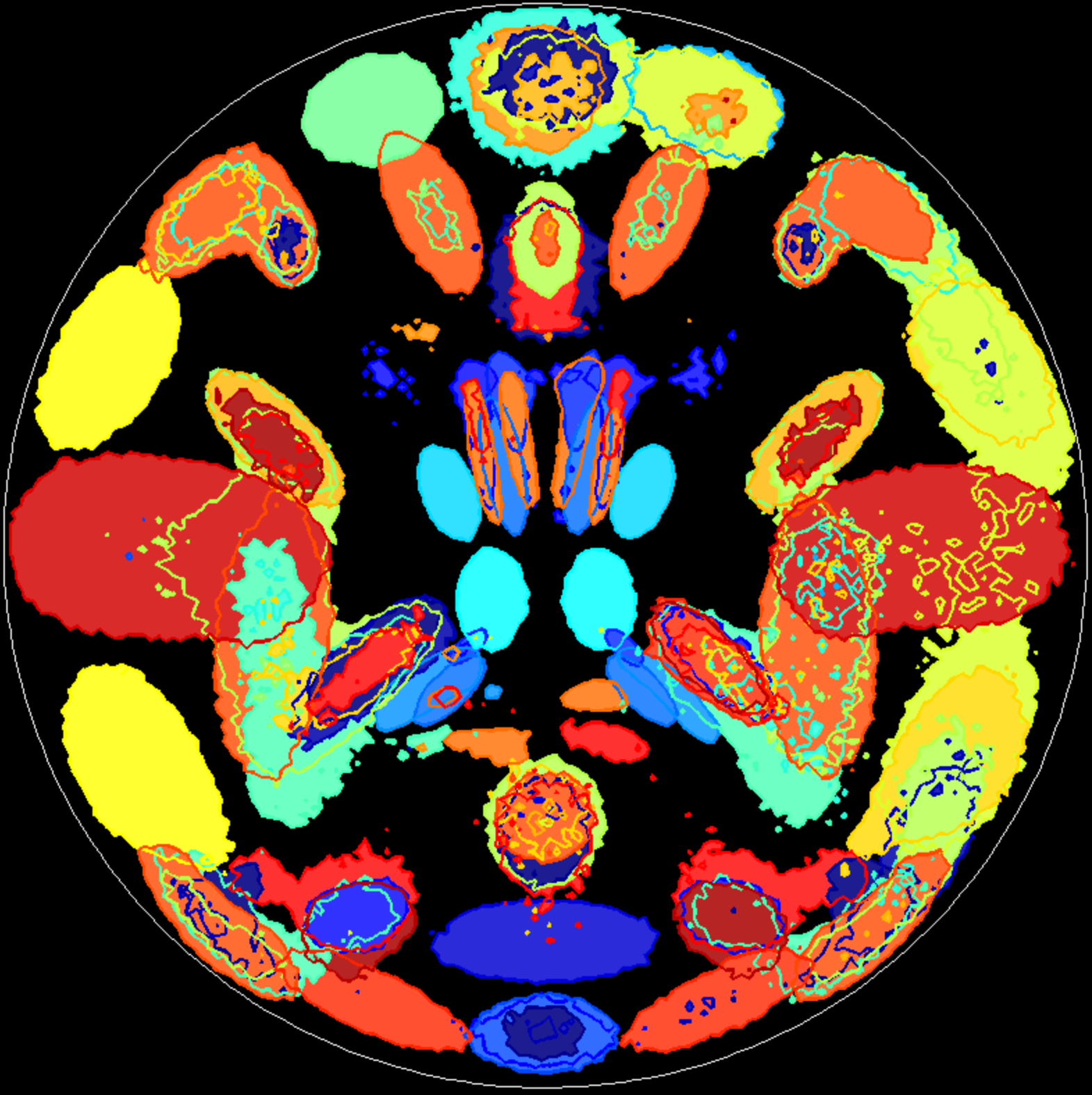} \\
   (a) \hspace{0.27 \columnwidth} (b) \hspace{0.27 \columnwidth} (c)
   \caption{Colors represent different spatial components. (a) The spatial map of 27 components used to generate fMRI data. (b) 27 spatial components recovered by the information sieve. (c) 27 spatial components recovered by ICA where components visualize the recovered mixing matrix.  }
   \label{fig:mri}
\end{figure}

%\note{ For the synthetic case, having $m$ independent Gaussian sources that each generate $k$ different $x_i$ with some Gaussian noise looks nice (i.e. a tree with one layer). The interesting thing about this case is that according to ICA, it shouldn't be possible (non-identifiable) but sieve picks out the ``correct'' sparse solutions anyway. (Why is this exactly!? My working intuition is that there is only one solution where individual latent factors can explain a lot of dependence by themselves. The other solutions are inherently synergistic in Y's; i.e., you need to know multiple latent factors)} 

\subsection{Dimensionality Reduction}

The sieve can be viewed as a dimensionality reduction (DR) technique. Therefore, we apply various DR methods to two standard datasets and use a Support Vector Machine with a Gaussian kernel to compare the classification accuracy after dimensionality reduction. The two datasets we studied were GISETTE and MADELON and consist of 5000 and 500 dimensions respectively. %Details about the datasets and the standard techniques used for comparison are in Sec.~\ref{sec:datasets}. 
For each method and dataset, we learn a low-dimensional representation on training data and then transform held-out test data and report the classification accuracy on that. The results are summarized in Fig.~\ref{fig:real}. 

\begin{SCfigure*}[][tbp]
    \centering
         (a)\includegraphics[width=0.75\columnwidth]{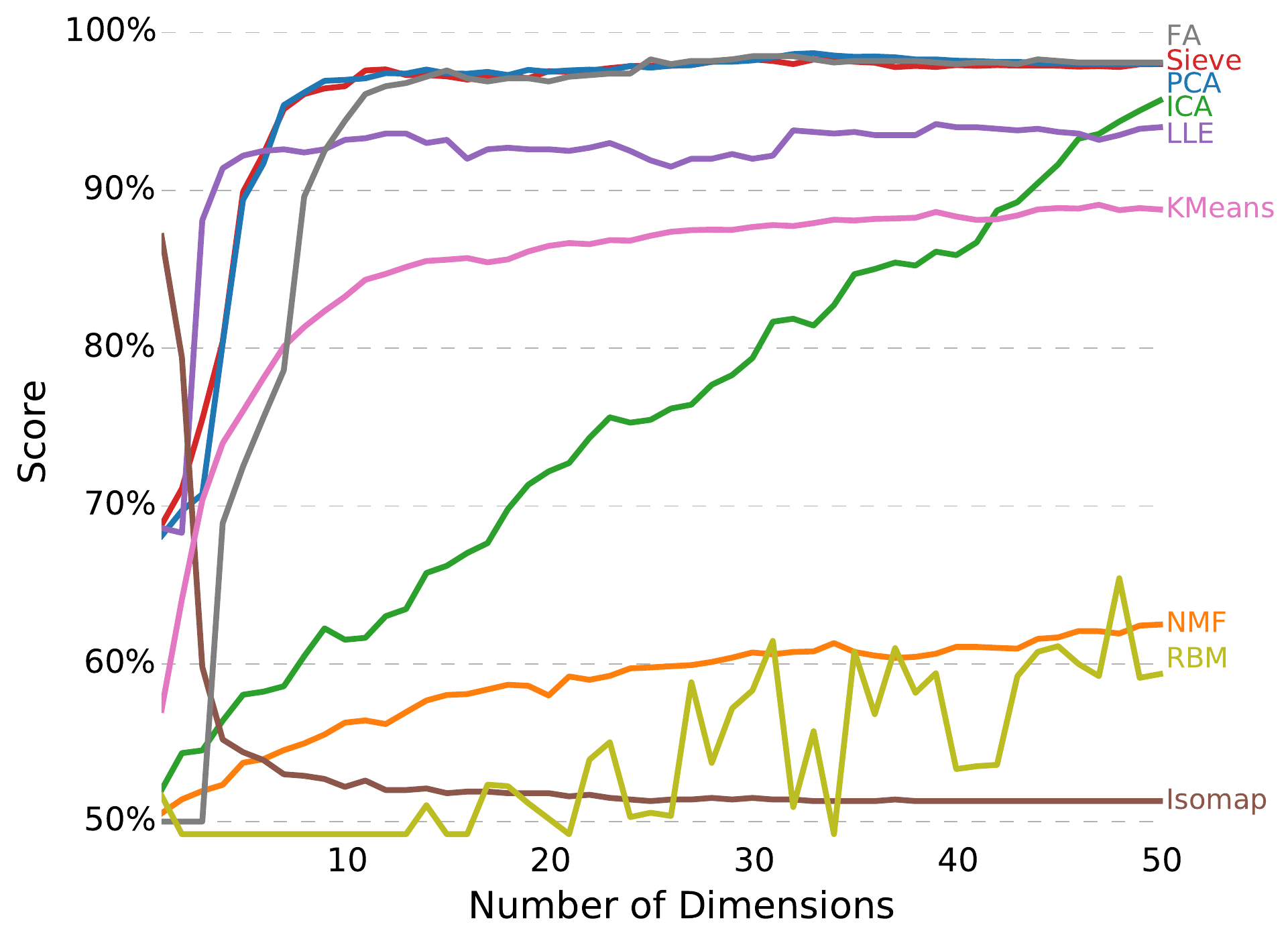}
          (b)\includegraphics[width=0.75\columnwidth]{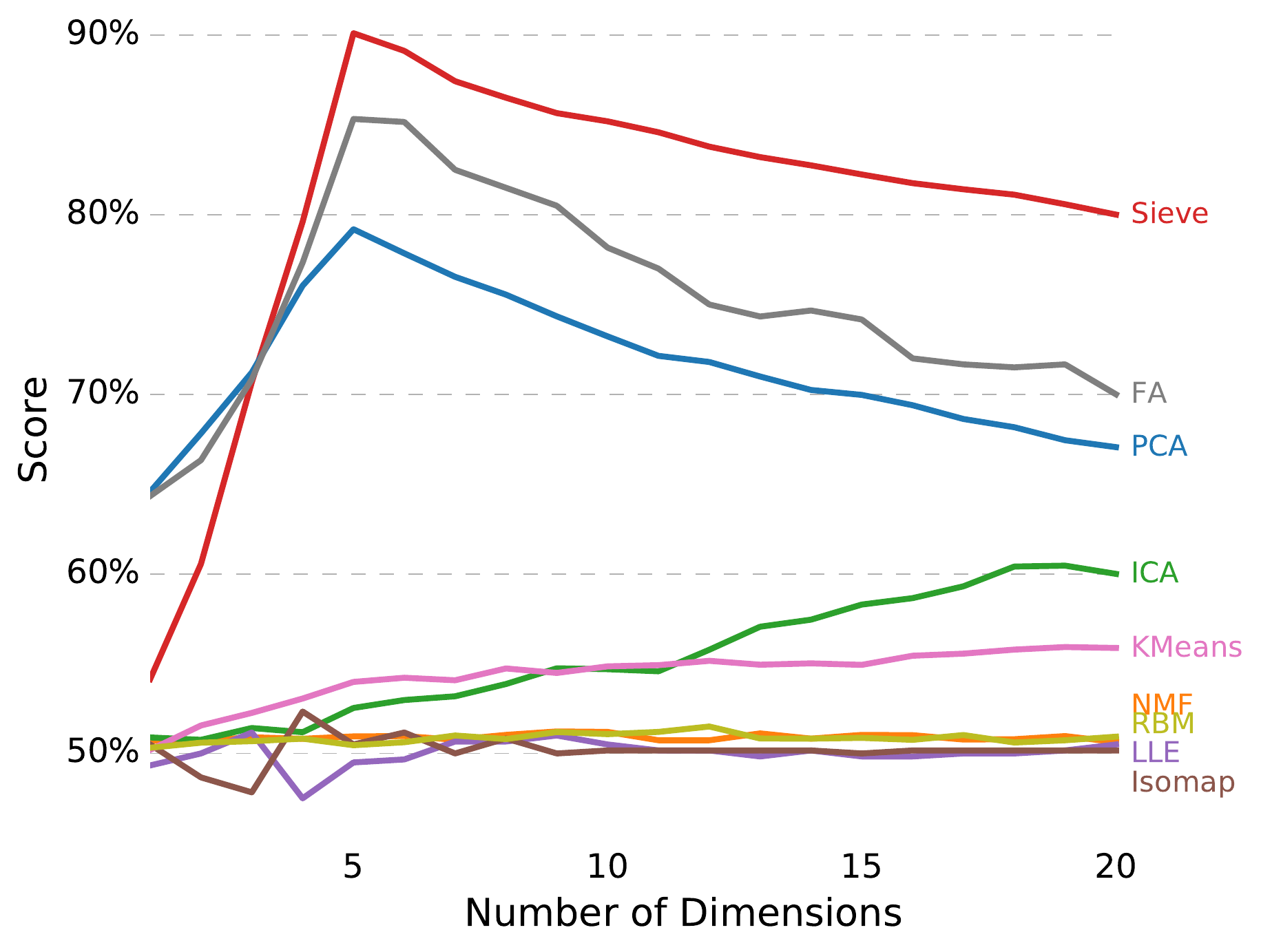}
\caption{ (a) Validation accuracy for GISETTE dataset (b) Validation accuracy for MADELON dataset. All the scores are averaged by running 20 trials.}\label{fig:real} \vspace{-2mm}
\end{SCfigure*}
For the GISETTE dataset, we see factor analysis, the sieve, and PCA performing the best, producing low dimensional representations with similar quality using a relatively small number of dimensions. 
For the MADELON dataset, the sieve representation gives the best accuracy with factor analysis and PCA resulting in accuracy drops of about five and ten percent respectively. Interestingly, all three techniques peak at five dimensions, which was intended to be the correct number of latent factors embedded in this dataset~\cite{guyon2004result}.

\section{Related Work}\label{sec:related}

Although the sieve is linear, the information objective that is optimized is nonlinear so the sieve substantially differs from methods like PCA. Superficially, the sieve might seem related to methods like Canonical Correlation Analysis (CCA) that seek to find a $Y$ that makes $X$ and $Z$ independent, but that method requires some set of labels, $Z$. One possibility would be to make $Z$ a copy of $X$, so that $Y$ is reducing dependence between $X$ and a copy of itself~\cite{fei_kernel}. However, this objective differs from common information as can be seen by considering the case where $X$ consists of independent variables. In that case the common information within $X$ is zero, but $X$ and its copy still have dependence. The concept of ``common information'' has largely remained restricted to information-theoretic contexts~\cite{xu_wyner,wyner_common,common,gastpar,gastpar2}. 
The common information in $X$ that is \emph{about} some variable, $Z$, is called intersection information and is also an active area of research~\cite{griffith_common}. 

Insofar as the sieve reduces the dependence in the data, it can be seen as an alternate approach to independent component analysis~\cite{comon} that is more directly comparable to ``least dependent component analysis''~\cite{lca}. 
As an information theoretic learning framework, the sieve could be compared to the information bottleneck~\cite{tishby}, which also has an interesting Gaussian counterpart~\cite{gib}. The bottleneck requires labeled data to define its objective. In contrast, the sieve relies on an unsupervised objective that fits more closely into a recent program for decomposing information in high-dimensional data~\cite{corex,corex_theory,sieve}, except that work focused on discrete latent factors.

The sieve could be viewed as a new objective for projection pursuit~\cite{friedman} based on common information. The sieve stands out from standard pursuit algorithms in two ways. First, an information based ``orthogonality'' criteria for subsequent projections naturally emerges and, second, new factors may depend on factors learned at previous layers (note that in Fig.~\ref{fig:sieve} each learned latent factor is included in the remainder information that is optimized over in the next step). 
More broadly, the sieve can be viewed as a new approach to unsupervised deep representation learning~\cite{bengioreview,hintonRBM}. In particular, our setup can be directly viewed as an auto-encoder with a novel objective~\cite{autoencoders}. From that point of view, it is clear that the sieve can also be directly leveraged for unsupervised density estimation~\cite{nice}. 

\section{Conclusion}\label{sec:conclusion}

We introduced a new scheme for incrementally extracting common information from high-dimensional data. 
The foundation of our approach is an efficient information theoretic optimization that finds latent factors that capture as much information about multivariate dependence in the data as possible. 
%As a linear decomposition, the sieve provides an intriguing complement to established methods like PCA and ICA. While PCA components explain the most variance, the sieve explains the most dependence. The sieve does not assume the existence of independent sources, but does produce factors that are progressively more independent at each step. Because the sieve decomposition also demands that each latent factor individually captures as much common information as possible, the sieve is able to uniquely recover sources that are unidentifiable from the point of view of ICA. 
With a practical method for extracting common information from high-dimensional data, we were able to explore new applications of common information in machine learning. Besides promising applications for exploratory data analysis and dimensionality reduction, common information seems to provide a compelling approach to blind source separation. 
 
While the results here relied on assumptions of linearity and Gaussianity, the invariance of the objective under nonlinear marginal transforms, a common ingredient in deep learning schemes, suggests a straightforward path to generalization that we leave to future work. The greedy nature of the sieve construction may be a limitation so another potential direction would be to jointly optimize several latent factors at once~\cite{nips2017}.
Sifting out common information in high-dimensional data provides a practical and distinctive new principle for unsupervised learning.

\section*{Acknowledgments}
GV thanks Sanjoy Dasgupta, Lawrence Saul, and Yoav Freund for encouraging exploration of the linear, Gaussian case for decomposing multivariate information. This work was supported in part by DARPA grant W911NF-12-1-0034 and IARPA grant FA8750-15-C-0071. 

{
\small
 \bibliographystyle{named}
\bibliography{gversteeg_bibdesk,shuyang} 
}

\end{document}